\documentclass{amsart}

\usepackage{amsmath,amssymb,amsthm,a4wide,caption,natbib}
\usepackage{pgfplots}
\usepackage{graphicx,tikz}

\newtheorem*{thm}{Theorem}

\newtheorem{lemma}{Lemma}

\newcommand{\diam}{\operatorname{diam}}

\theoremstyle{definition}

\theoremstyle{remark}

\newcommand{\conv}{\operatorname{conv}}

\makeatletter
\def\@setthanks{\vspace{-\baselineskip}\def\thanks##1{\@par##1\@addpunct.}\thankses}
\makeatother

\begin{document}

\title[]{Clustering with t-SNE, provably.}

\author[]{George C. Linderman}
\thanks{GCL was supported by NIH grant \#1R01HG008383-01A1 (PI: Yuval Kluger) and U.S. NIH MSTP Training Grant T32GM007205. SS was partially supported by \#INO15-00038 (Institute of New Economic Thinking).}
\address{Program in Applied Mathematics, Yale University, New Haven, CT 06511, USA}
\email{george.linderman@yale.edu}

\author[]{Stefan Steinerberger}
\address{Department of Mathematics, Yale University, New Haven, CT 06511, USA}
\email{stefan.steinerberger@yale.edu}

\begin{abstract} 
t-distributed Stochastic Neighborhood Embedding (t-SNE), a clustering
and visualization method proposed by van der Maaten \& Hinton in 2008,
has rapidly become a standard tool in a number of natural
sciences. Despite its overwhelming success, there is a distinct lack of
mathematical foundations and the inner workings of the algorithm are
not well understood. The purpose of this paper is to prove that t-SNE
is able to recover well-separated clusters; more precisely, we prove
that t-SNE in the `early exaggeration' phase, an optimization technique proposed
by van der Maaten \& Hinton (2008) and van der Maaten (2014), can be
rigorously analyzed.
As a byproduct, the proof suggests novel ways for setting the
exaggeration parameter $\alpha$ and step size $h$. Numerical examples
illustrate the effectiveness of these rules: in particular, the quality
of embedding of topological structures (e.g. the swiss roll) improves.
We also discuss a connection to spectral clustering methods.
\end{abstract}

\maketitle

\section{Introduction and main result}
\subsection{Introduction.}The analysis of large, high dimensional datasets is
ubiquitous in an increasing number of fields and vital to their progress.
Traditional approaches to data analysis and visualization often fail in the
high dimensional setting, and it is common to perform dimensionality reduction
in order to make data analysis tractable.  t-distributed Stochastic
Neighborhood Embedding (t-SNE), introduced by \cite{maaten2008visualizing}, is
an impressively effective non-linear dimensionality reduction technique that
has recently found enormous popularity in several fields.  It is most commonly
used to produce a two-dimensional embedding of high dimensional data with the
goal of simplifying the identification of clusters.  Despite its tremendous
empirical success, the theory underlying t-SNE is unclear. The only theoretical
paper at this point is \cite{uri}, which shows that the structure of the loss
functional of SNE (a precursor to t-SNE) implies that global minimizers
separate clusters in a quantitative sense.

\subsection{A case study.}  As an unsupervised learning method, t-SNE is
commonly used to visualize high dimensional data and provide crucial intuition
in settings where ground truth is unknown.  The analysis of single cell RNA
sequencing (scRNA-seq) data, where t-SNE has become an integral part of the
standard analysis pipeline, provides a relevant example of its usage. 
\begin{center}
\begin{figure}[h!]
\includegraphics[width=0.75\textwidth]{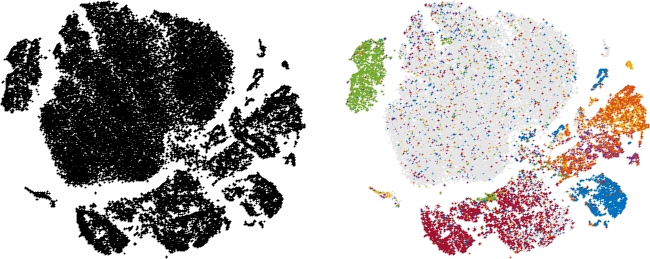}
\caption{t-SNE output (left) and colored by some known ground truth (right).}
\label{fig:case_study}
\end{figure}
\end{center}
Figure \ref{fig:case_study} shows (left) the output of running t-SNE on the 30 largest principal components of the normalized
expression matrix of 49300 retinal cells taken from \cite{macosko2015highly}.
The output on the right has cells colored based on which of 12 cell type marker genes were
most expressed (with grey signifying that none of the marker genes were
expressed). This example is well suited to showcase both the tremendous impact
of t-SNE in the medical sciences as well as the inherent difficulties of
interpreting its output when ground truth is unknown: how many clusters
are in the original space, and do they correspond one-to-one to clusters in
the t-SNE plot? Do the clusters (e.g. the largest cluster that does not express
any marker genes) have substructure that is not apparent in this visualization?
Pre-processing steps will yield different embeddings; how stable are the
clusters? All these questions are of the utmost importance and underline the
need for a better theoretical understanding.  

\subsection{Early Exaggeration} 
t-SNE (described in greater
detail in \S\ref{sec:tsne}) minimizes the
Kullback-Leibler divergence between a Gaussian distribution modeling distances between points in
the high dimensional input space and a Student t-distribution modeling
distances between corresponding points in a low dimensional embedding.
Given a $d$-dimensional input dataset $\mathcal{X} = \{ x_1, 
x_2, ...,  x_n \} \subset \mathbb{R}^d$,  t-SNE computes an $s$-dimensional embedding of the
points in $\mathcal X$, denoted by $\mathcal Y = \{  y_1,  y_2, ...,  y_n
\} \subset \mathbb{R}^s$, where $s \ll d$ and most commonly $s= 2 \text{ or } 3$. The main idea is to
define a series of affinities $p_{ij}$ on $\mathcal{X}$ as well as a series of
affinities $q_{ij}$ in the embedding $\mathcal{Y}$ and then try minimize the distance
of these distributions in the Kullback-Leibler distance
$$C(\mathcal Y) = KL(P || Q) = \sum_{i\neq j} p_{ij} \log \frac{p_{ij}}{q_{ij}},$$
which gives rise to a gradient descent method via
$$\frac{\partial C}{\partial  y_i } = 4 \sum\limits_{j \neq i} (p_{ij} - q_{ij})q_{ij} Z ( y_i -  y_j).$$
One difficulty is that the convergence rate slows down as the number of points $n$ increases. However,
already the original paper \cite{maaten2008visualizing} proposes a number of
ways in which the convergence can be accelerated.
\begin{quote}
A less obvious way to improve the optimization, which we call `early exaggeration', is to
multiply all of the $p_{ij}$'s by, for example, 4, in the initial stages of the optimization. [...]
In all the visualizations presented in this paper and in the supporting material, we used exactly
the same optimization procedure. We used the early exaggeration method with an exaggeration
of 4 for the first 50 iterations (note that early exaggeration is not included in the pseudocode in
Algorithm 1). (from: \cite{maaten2008visualizing})
\end{quote}

It is easy to test empirically that this renormalization indeed improves the
clustering and is effective. It has become completely standard and is
hard-coded into the very widely used standard implementation available online,
as described by \cite{van2014accelerating}:

\begin{quote}
During the first 250 learning iterations, we multiplied all $p_{ij}-$values by a user-defined
constant $\alpha > 1$. [...] this trick enables
t-SNE to find a better global structure in the early stages of the optimization by creating
very tight clusters of points that can easily move around in the embedding space. \textit{In
preliminary experiments, we found that this trick becomes increasingly important to obtain
good embeddings when the data set size increases} [Emphasis GL \& SS], as it becomes harder for the optimization
to find a good global structure when there are more points in the embedding because there
is less space for clusters to move around. In our experiments, we fix $\alpha = 12$ (by contrast,
van der Maaten and Hinton (2008) used $\alpha=4$). (from: \cite{van2014accelerating})
\end{quote}

As it turns out, this simple optimization trick can be rigorously analyzed.
 As a byproduct of our analysis, we
see that the convergence of non-accelerated t-SNE will slow down as the number of points $n$ increases
and the number of iterations required will grow at least linearly in $n$. The implementation
available online counteracts this problem by various methods: (1) the early exaggeration factor $\alpha$,
(2) a large ($h=200$) stepsize in the gradient descent
$$ y_i(t+1) = y_i(t) - h\frac{\partial C}{\partial y_i(t)},$$ and by (3)
optimization techniques such as momentum. We only deal with the t-SNE
algorithm, the early exaggeration factor $\alpha$ and the step-size $h$; one of
the main points of our paper is that a suitable parameter selection of $\alpha$ and $h$ 
makes it possible to guarantee fast convergence without additional
optimization techniques.

\subsection{Summary of Main Results.} We will now state our main results at an informal level; all the
statements can be made precise (this is done in \S \ref{subsec:main_result}) and will be rigorously proven. 
\begin{enumerate}
\item \textit{Canonical parameters and exponential convergence.} 
There is a canonical setting for the parameters $\alpha, h$ for which
the algorithm applied to clustered data converges provably at an
exponential rate without the use of other optimization techniques (such
as momentum). This setting is
$$ \alpha \sim \frac{n}{10} \qquad \mbox{and} \qquad h \sim 1.$$
These parameters lead to an exponential convergence of all embedded clusters to small balls (whose diameter
depends on how well $\mathcal{X}$ is clustered).  Generally, the speed of convergence is exponential with an exponential factor $\kappa$
$$\kappa \sim 1 - \frac{\alpha h}{n}.$$
Moreover, if $\alpha h \gtrsim n$, then convergence of the algorithm breaks down. This theoretical
result is actually applicable to the early exaggeration phase of the classical t-SNE implementation as long as the number of points is not too
large (roughly, $n \lesssim 20000$).

\item \textit{Spectral clustering.} The t-SNE algorithm, in this regime,
	behaves like a spectral clustering algorithm; moreover, this algorithm
	can be written down explicitly.  This allows for (1) the use of theory
	from spectral clustering to rigorously analyze t-SNE and (2) a
	fast implementation that can perform the early exaggeration phase in a
	fraction of the time necessary to run t-SNE (in this regime). It also poses the
	challenge of trying to understand whether t-SNE behaves qualitatively
	different for the standard parameters $\alpha \sim 12, h \sim 200$ or
	whether it behaves more or less identically (and thus like a spectral
	method). 

\item \textit{Disjoint clusters.}
	It is not guaranteed that the embedded clusters in $\mathcal{Y}$ are
	disjoint; but given a random initialization, it is extremely unlikely
	that two distinct clusters will converge to the same center.
	Furthermore, if $\mathcal{X}$ is well-clustered, the diameter of the
	clusters $\mathcal{Y}$ can be made even smaller by decreasing the
	step-size $h$ and further increasing $\alpha$ as long as the product
	satisfies $\alpha h \sim n/10$. Increasing $\alpha$ will resolve overlapping
	clusters, as long as they have different centers.  In particular, the
	number of disjoint clusters in $\mathcal{Y}$ is a lower bound on the
	number of clusters in $\mathcal{X}$ (and, generically, the numbers coincide).

\item \textit{Independence of initialization.} All these results are independent of the initialization of $\mathcal{Y}$ as long as it is contained in a sufficiently
small ball.
\end{enumerate}

An immediate implication of (3) is the following: if we are given some
clustered data $\mathcal{X}$ and see that the embedding of t-SNE for large
values of $\alpha$ (and small values of $h$) produces $k$ clusters, then there
are exactly $k$ clusters in $\mathcal{X}$. The results guarantee that all
clusters in $\mathcal{X}$ are eventually mapped to small balls which can be
made arbitrarily small.  We see that when parameters are chosen optimally, this
result provides a justification for the way t-SNE is commonly used in, say,
biomedical research.  
\subsection{Approximating Spectral Clustering.} The fact that t-SNE approximates a spectral 
clustering method for $\alpha \sim n/10, ~h \sim 1$ raises a fascinating question: does t-SNE, in its early exaggeration
phase, perform better with the classical parameter choices of $\alpha \sim 12, h \sim 200$ than it does
with $\alpha \sim n/10,~ h \sim 1$? If yes, then its inner workings may give rise to improved spectral
methods. If no, then it would be advantageous to use $\alpha \sim n/10, ~h \sim 1$, which then, however,
is essentially a spectral method and it may be advantageous (and much faster) to initialize the second phase of t-SNE
by using the outcome of a more advanced spectral method as initialization instead. We discuss some 
experiments in that direction in \S\ref{subsec:spectral} and believe this to be worthy of further investigation. Moreover, we
 describe a visualization technique in the style of t-SNE for spectral clustering tools (see \S \ref{sec:visualizing_spec_clust}).

\subsection{Organization.} The Organization of this paper is as follows: we
first illustrate our main points with some numerical examples in Section \S
\ref{sec:numerical}. Section \S \ref{sec:tsne} establishes notation and a
formal statement of our main result,  Section \S \ref{sec:spectral} derives a
connection between t-SNE and spectral clustering, Section \S
\ref{sec:dynamical} discusses a certain type of discrete dynamical system on
finite numbers of points and establishes a crucial estimate, Section \S
\ref{sec:proof_main_result} gives a proof of the main result.

\section{Numerical examples}\label{sec:numerical}
This section discusses a number of numerical examples to illustrate our main points.

\subsection{Lines and Swiss roll.} 
It is classical that t-SNE does not successfully embed the swiss roll; however, the random initialization
causes difficulty even on simpler data: Figure \ref{fig:classical_tsne} shows the t-SNE embedding
(using Matlab implementation of \cite{van2014accelerating} with default parameters) of a
simple line in $\mathbb{R}^3$.
\begin{center}
\begin{figure}[h!]
\begin{tikzpicture}[scale=1]
\node at(0,0) {\includegraphics[width=0.8\textwidth]{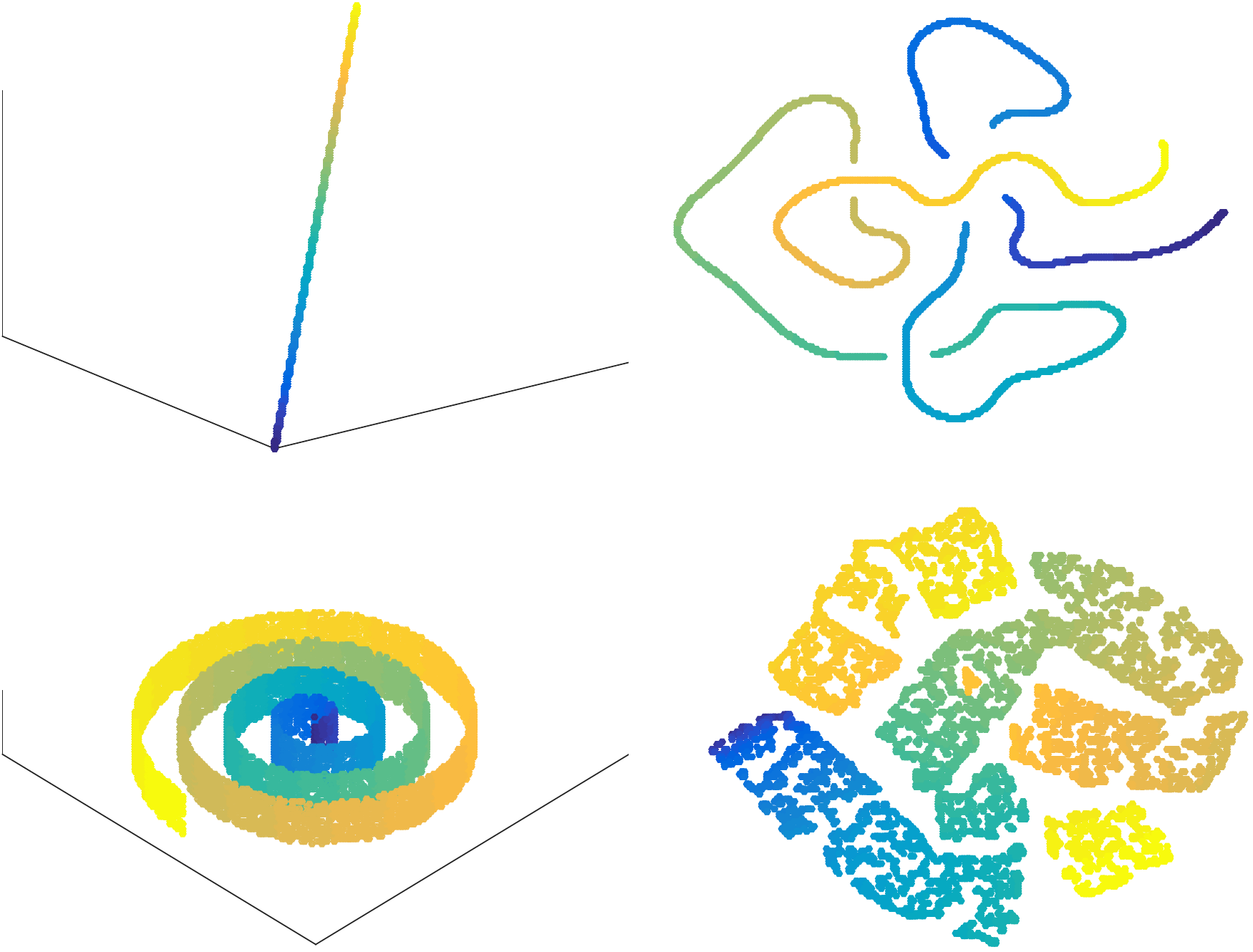}};
\node at (-3,5) {Input: a line in $\mathbb{R}^3$};
\node at (-3,-1) {Input: swiss roll in $\mathbb{R}^3$};
\node at (3,5) {t-SNE output in $\mathbb{R}^2$};
\node at (3,0) {t-SNE output in $\mathbb{R}^2$};
\end{tikzpicture}
\caption{Classical t-SNE embeddings of a line and the swiss roll.}
\label{fig:classical_tsne}
\end{figure}
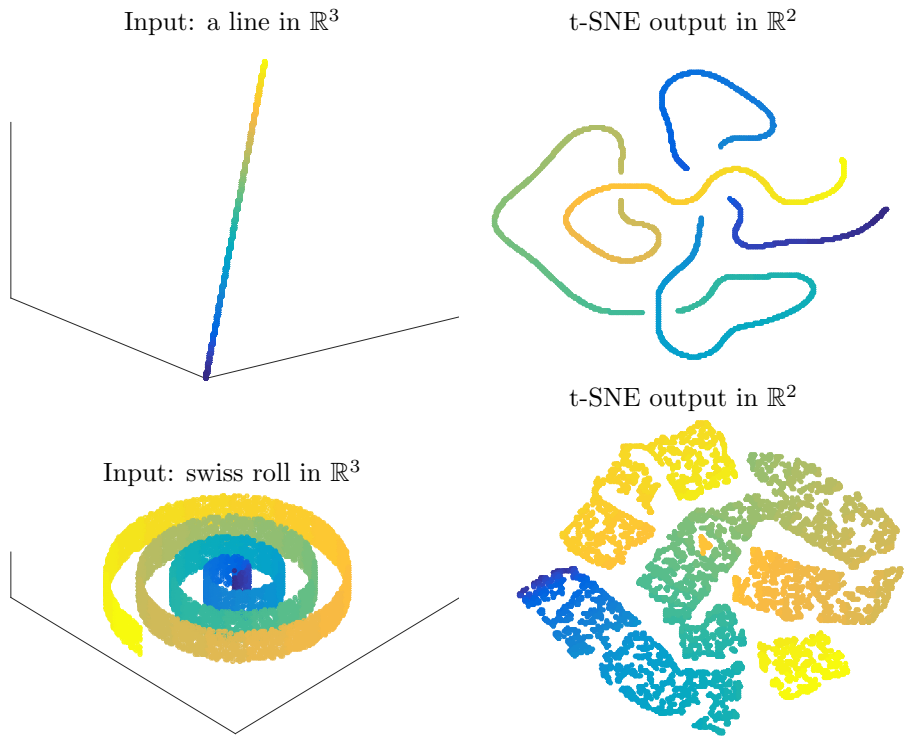
\end{center}

The randomized initialization causes, after initial contraction in the early
exaggeration phase, a topological interlocking that cannot be further resolved.
The example is even more striking with the swiss roll, where the random
initialization leads to `knots' that cannot be untied by t-SNE. In stark
contrast, the parameter selection $$ \alpha \sim \frac{n}{10} \qquad \mbox{and}
\qquad h \sim 1$$ allows for a more effective early exaggeration phase that
clearly recovers the line from random initial data and even contracts the swiss
roll to a correctly ordered line (that would then expand in the second phase of
the algorithm).  

\begin{center}
\begin{figure}[h!]
\begin{tikzpicture}
\node at(0,0) {\includegraphics[width=\textwidth]{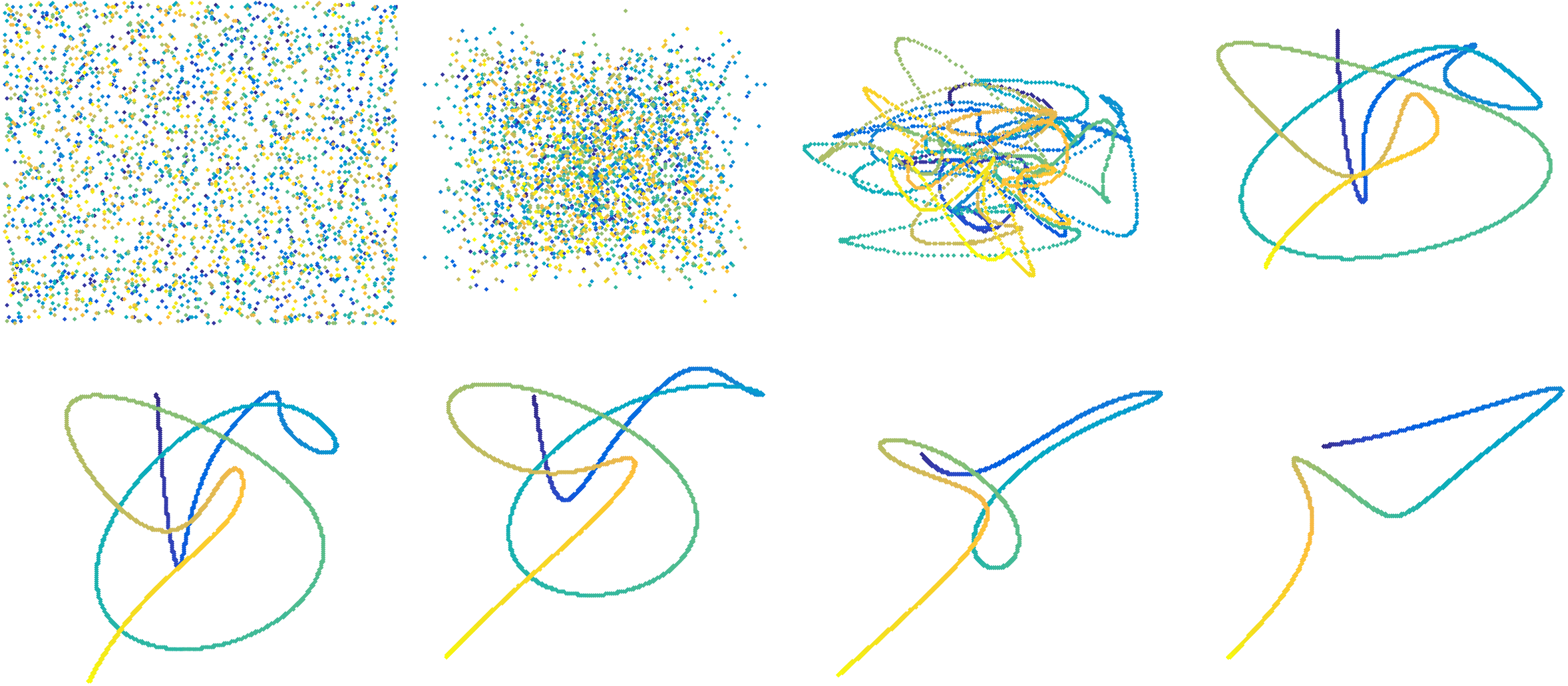}};
\node at (-5.6,3.5) {Initialization};
\node at (-1.8,3.5) {Step 1};
\node at (2,3.5) {Step 5};
\node at (5.5,3.5) {Step 200};
\node at (-5.6,-0.2) {Step 300};
\node at (-1.8,-0.2) {Step 500};
\node at (2.3,-0.2) {Step 1000};
\node at (5.5,-0.2) {Step 2000};
\end{tikzpicture}
\caption{Early exaggeration phase of t-SNE on a line with $\alpha = 20n, h = 0.05$.}
\label{fig:line}
\end{figure}
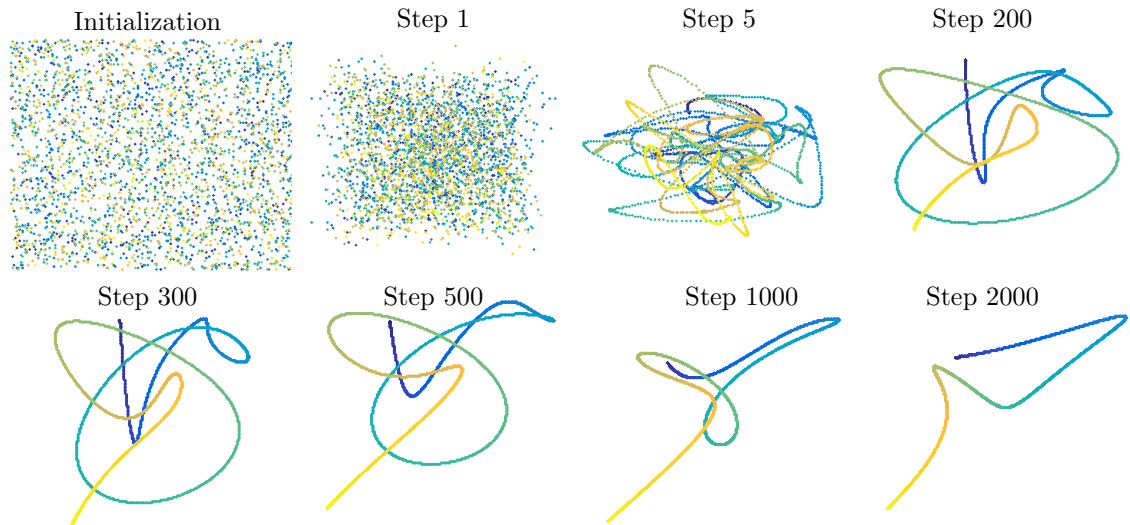
\end{center}
\vspace{-20pt}

\begin{center}
\begin{figure}[h!]
\begin{tikzpicture}
\node at(0,0) {\includegraphics[width=\textwidth]{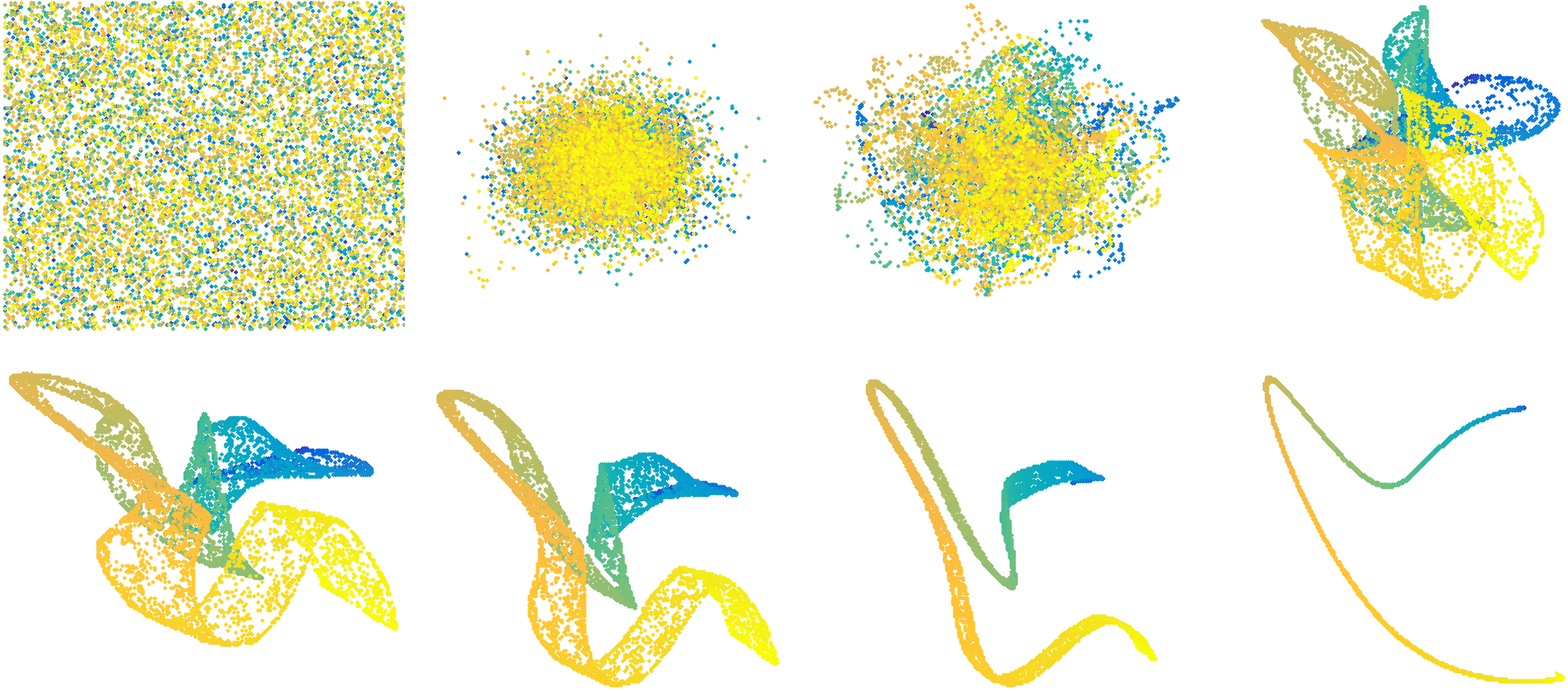}};
\node at (-5.6,3.5) {Initialization};
\node at (-1.8,3.5) {Step 1};
\node at (2,3.5) {Step 5};
\node at (5.5,3.5) {Step 75};
\node at (-5.6,-0.2) {Step 150};
\node at (-1.8,-0.2) {Step 200};
\node at (2.3,-0.2) {Step 300};
\node at (5.5,-0.2) {Step 900};
\end{tikzpicture}
\caption{Early exaggeration phase of t-SNE on a swiss roll with $\alpha = n, h = 1$.}
\label{fig:swiss}
\end{figure}
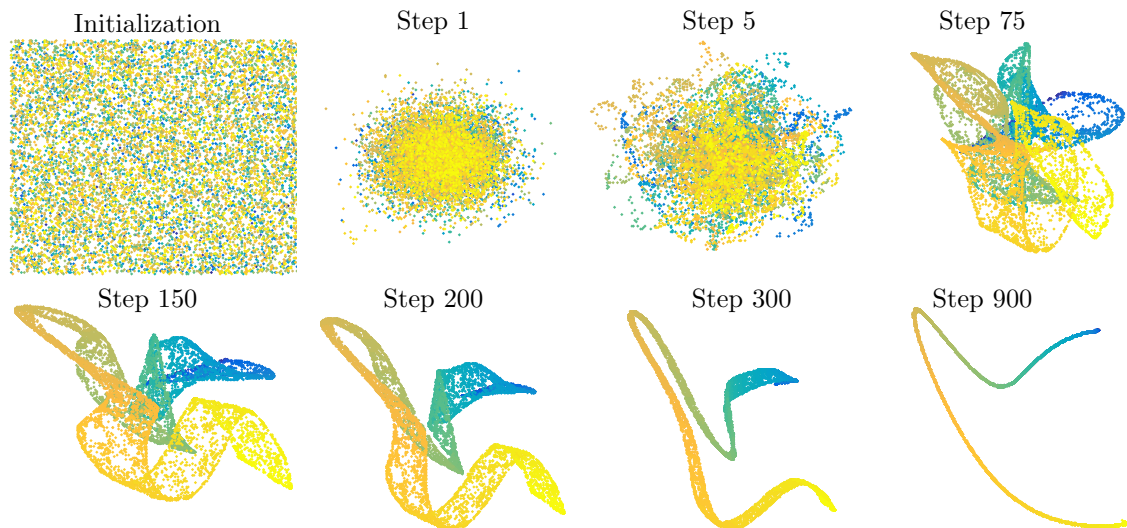
\end{center}
\vspace{-20pt}
The successful embedding of these examples when $\alpha$ and $h$ are chosen
optimally is consistent with our claim that in this regime, the early
exaggeration phase of t-SNE acts like a spectral method, many of which also
correctly embed these manifolds.

\subsection{Real-life data.}
Finally, we show the impact of the parameter selection $\alpha \sim n/10, h
\sim 1$ in a real-life example. Figure \ref{fig:mnist1} shows (left) classical out-of-the-box
t-SNE on 10000 randomly subsampled handwritten digits (0--5) from the MNIST dataset as well as the outcome of
the early exaggeration phase of t-SNE with parameters  $\alpha \sim n/10, h
\sim 1$ (middle) and the final outcome after the second phase of t-SNE has been
initialized with the data shown in the middle (right). We see that early exaggeration
does essentially all the clustering already and the second phase rearranges them.

\begin{center}
\begin{figure}[h!]
\begin{tikzpicture}
\node at(0,0) {\includegraphics[width=\textwidth]{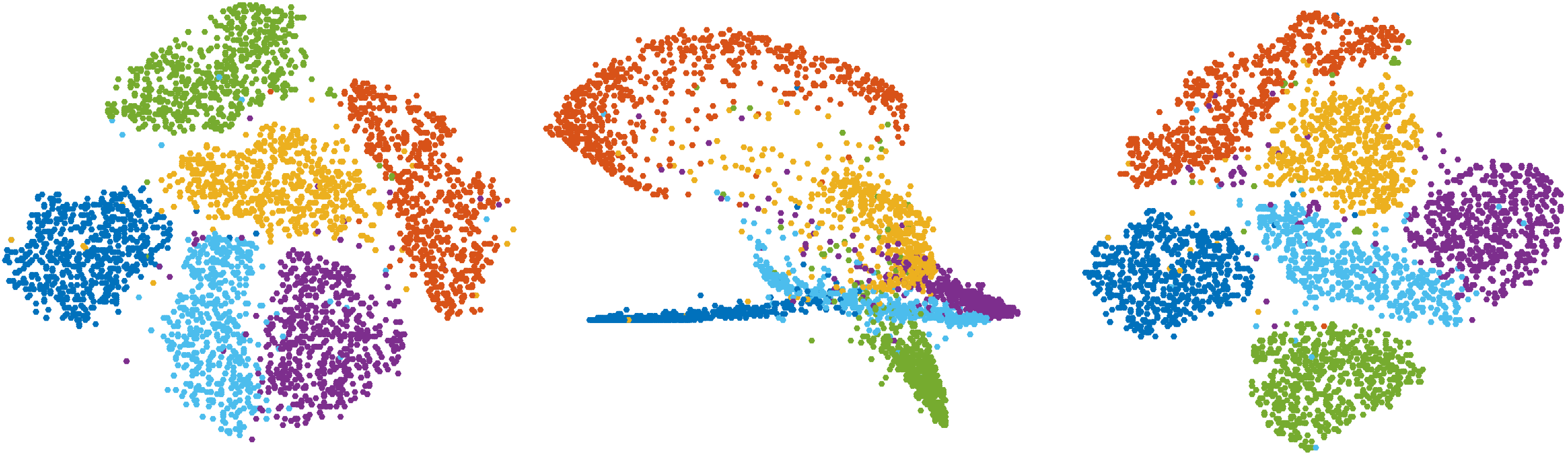}};
\node at (-5,3) {Classical t-SNE};
\node at (-0.5,3) {t-SNE with $\alpha = 0.1n,h = 1$};
\node at (-0.5,2.5) {after early exaggeration phase};
\node at (4.9,3) {t-SNE with $\alpha = 0.1n,h = 1$};
\node at (4.9,2.5) {final output};
\end{tikzpicture}
\caption{A real-life example: classical t-SNE (left), t-SNE with our proposed parameter selection both after
the early exaggeration phase (middle) and final output (right). }
\label{fig:mnist1}
\end{figure}
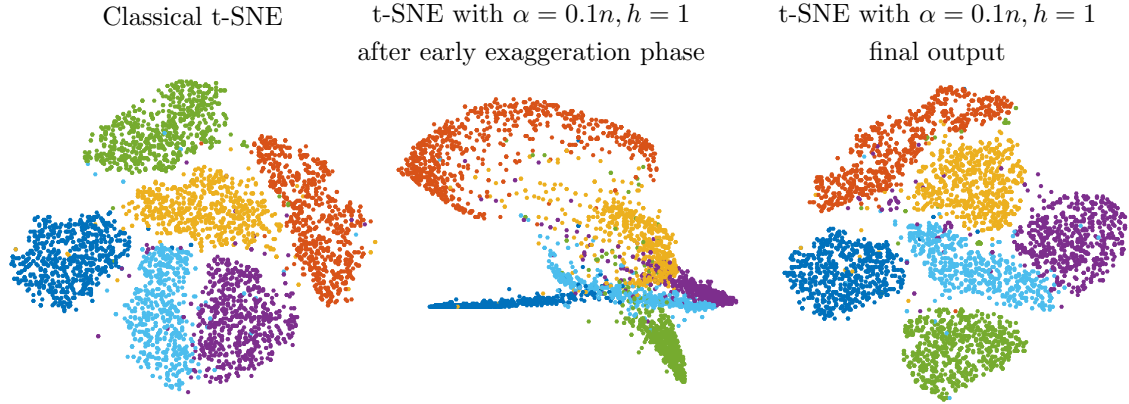
\end{center}

 We believe this example again hints at one of the fundamental questions that
 arises from the work in this paper: is the initial clustering done by standard
 t-SNE comparable to the initial clustering with the new parameter selection?
 If so, then the fact that the new parameter selection emulates a spectral
 clustering method (see \S \ref{sec:spectral}) certainly suggests the option of
 initializing with other clustering methods as opposed to random
 initialization. Moreover, it would hint at the danger of using a spectral
 clustering method and t-SNE as a dual verification of clustering.

\section{t-SNE: Notation and the Main result}\label{sec:tsne}

\subsection{t-SNE}
We denote the $d$-dimensional input dataset by $\mathcal{X} = \{x_1, 
x_2, ...,  x_n \} \subset \mathbb{R}^d$,  t-SNE computes an $s$-dimensional embedding of the
points in $\mathcal X$, denoted by $\mathcal Y = \{ y_1,  y_2, ...,  y_n
\} \subset \mathbb{R}^s$, where $s \ll d$ and most commonly $s= 2 \text{ or } 3$.  The joint
probability $p_{ij}$ measuring the similarity between $x_i$ and $x_j$ is
computed as:
$$p_{i|j} = \frac{\exp{(-\|  x_i -  x_j \|^2 /2 \sigma_i^2 )}}{\sum_{k\neq i} \exp{( - \|  x_i -  x_k \|^2/ 2 \sigma_i^2 }) } \qquad \mbox{and} \qquad \label{p_ij}
	p_{ij}  = \frac{ p_{i|j} + p_{j|i} }{2n}.$$
The bandwidth of the Gaussian kernel, $\sigma_i$, is often chosen such that the
perplexity of $P_i$ matches a user defined value, where $P_i$ is the
conditional distribution across all data points given $x_i$. We will never deal
with these issues: we will assume that the $p_{ij}$ are given and that they correspond
to a well-clustered set $\mathcal{X}$ (in a precise sense defined below). In particular,
we will not assume that they have been obtained using a Gaussian kernel.
The similarity between points $ y_i$ and $ y_j$ in the low
dimensional embedding is defined as:
$$q_{ij} = \frac{(1 + \| y_i -  y_j\|^2)^{-1}}{\sum_{k\neq l} (1 +\|  y_k -  y_l\|^2 )^{-1}}.$$
t-SNE finds the points $\{y_1, \dots, y_n \}$ which minimize the
Kullback-Leibler divergence between the joint distribution $P$ of points in the
input space and the joint distribution $Q$ of points in the embedding space:
$$C(\mathcal Y) = KL(P || Q) = \sum_{i\neq j} p_{ij} \log \frac{p_{ij}}{q_{ij}}.$$
The points $\mathcal Y$ are initialized randomly, and the cost function
$C(\mathcal Y)$ is minimized using gradient descent.  The gradient is derived
in Appendix A of \cite{maaten2008visualizing}:
$$\frac{\partial C}{\partial  y_i } = 4 \sum\limits_{j \neq i} (p_{ij} - q_{ij})q_{ij} Z ( y_i -  y_j),$$
where $Z$ is a global normalization constant
$$Z = \sum_{k\neq l}{ \left( 1 + \| y_k -
 y_l \|^2\right)^{-1}}.$$ 
  As in \cite{van2014accelerating}, we split the
gradient into two parts:
$$	\frac{1}{4} \frac{\partial C}{\partial  y_i } =  \sum_{j\neq i} {p_{ij} q_{ij} Z (  y_i -  y_j)} - \sum_{j\neq i}{ q_{ij}^2 Z ( y_i - y_j)} $$
where the first and second sums correspond to the sum of all attractive forces
and the sum of all repulsive forces, respectively. Early exaggeration introduces the coefficient $\alpha >1$ and 
corresponds to the gradient descent method
$$	\frac{1}{4} \frac{\partial C}{\partial  y_i } =  \sum_{j\neq i} { \alpha p_{ij} q_{ij} Z (  y_i -  y_j)} - \sum_{j\neq i}{ q_{ij}^2 Z ( y_i - y_j)} $$
and a small step-size $h > 0$ leads to the expression
$$	\frac{h}{4} \frac{\partial C}{\partial  y_i } =  h\sum_{j\neq i} {  \alpha   p_{ij} q_{ij} Z (  y_i -  y_j)} - h\sum_{j\neq i}{ q_{ij}^2 Z ( y_i - y_j)}.$$

\subsection{Main result}\label{subsec:main_result}
This section gives our main result. We emphasize that the method of proof is rather flexible and it is not difficult to obtain variations on the result under slightly
different assumptions. We emphasize that our result is formally stated for a set of points $\left\{x_1, \dots, x_n \right\}$ and a set of mutual affinities $p_{ij}.$
We will not assume that the $p_{ij}$ are obtained using the standard t-SNE normalizations but work at a full level of generality using a set of three assumptions.
We note, and explain below, that for standard t-SNE the second assumption holds until the number
of points exceeds, roughly, $n \sim 20000$ and the third assumption holds by design. The first assumption encapsulates our notion of clustered data.\\

\textit{1. $\mathcal{X}$ is clustered.} We proceed by giving a very versatile definition of what it means to be a cluster; it is trivially applicable to things that
clearly are not clusters, however, in those cases the error bound in the Theorem will not convey any information. Formally, we assume that there exists a 
$k \in \mathbb{N}$ (the number of clusters) and a map $\pi:\left\{1, \dots, n\right\} \rightarrow \left\{1,2, \dots, k \right\}$ assigning each point to one of the $k$
clusters such that the following property holds: if $\pi(x_i) = \pi(x_j),$ then
$$ p_{ij} \geq  \frac{1}{10 n | \pi^{-1}(\pi(i))| }.$$
Observe that $ |\pi^{-1}(\pi(i))|$ is merely the size of the cluster in which $i$ and $j$ lie.
We will furthermore abbreviate, for fixed $1 \leq i \leq n$, summations over
clusters as
$$  \sum_{j \neq i \atop \mbox{\tiny same cluster}} :=  \sum_{j \neq i \atop \pi(j) = \pi(i)} \qquad \mbox{and} \qquad 
  \sum_{j \neq i \atop \mbox{\tiny other clusters}} :=  \sum_{j \neq i \atop \pi(j) \neq \pi(i)}$$

\textit{2. Parameter choice.} We assume that $\alpha$ and $h$ are chosen such that, for some $1 \leq i \leq n$
 $$  \frac{1}{100} \leq     \alpha h \sum_{j \neq i \atop \mbox{\tiny same cluster}}{p_{ij}} \leq \frac{9}{10}.$$

The main result will be applicable to single cluster (i.e. it is possible to guarantee that a single cluster converges even if
the rest does not) and it can be applied to exactly those clusters satisfying this inequality.
It is easy to see, both in the proof and in numerical experiments, that the upper bound is a necessary condition for the early exaggeration phase of
t-SNE to work (more precisely, the upper bound 1 is necessary but we need a little bit of leeway in another part of the argument).  We observe that condition (1) implies that $\alpha \sim n/10$ and $h \sim 1$ is admissible, however, other parameter choices
(i.e. $\alpha \sim 10n, h \sim 1/100$) are equally valid. In particular, for a small number of points (roughly $n \lesssim 24000$), the standard t-SNE parameter selection $\alpha \sim 12,~h\sim 200$
does satisfy these bounds. If the number of points gets larger, the lower bound is violated: our main result can be easily extended to cover that case,
however, the factor $\kappa$ with which exponential convergence occurs approaches 1 and convergence, while technically exponential, slows down.
In particular, an analysis of how this condition acts in the proof motivates an accurate parameter selection rule.

\begin{quote}
\textit{Guideline.} The best convergence rate for the cluster containing $y_i$ is attained when
 $$\alpha h = \frac{9}{10} \left( \sum_{j \neq i \atop \mbox{\tiny same cluster}}{p_{ij}}\right)^{-1} \quad \mbox{while} \quad  \alpha h  =\frac{9}{10} \left( \max_{1 \leq i \leq n}{ \sum_{j \neq i \atop \mbox{\tiny same cluster}}{p_{ij}}}\right)^{-1}$$
is the best selection to ensure that all clusters converge.
\end{quote}

 \textit{3. Localized initialization.} The initialization satisfies $\mathcal{Y} \subset [-0.01,0.01]^2$. This assumption is not crucial and could be easily modified at
the cost of changing some other constants. The proof suggests that initializing at smaller scales might be 
beneficial on the level of constants.

\begin{thm} The diameter of the embedded cluster $ \left\{y_j: 1 \leq j \leq n \wedge \pi(j) = \pi(i) \right\}$
decays exponentially (at universal rate) until its diameter satisfies, for some universal constant $c>0$,
$$ \diam  \left\{y_j: 1 \leq j \leq n \wedge \pi(j) = \pi(i) \right\}  \leq   c \cdot h \left( \alpha  \sum_{j \neq i \atop \mbox{\tiny other clusters}}{ p_{ij} } + \frac{1}{n}\right).$$
\end{thm}
\textit{Remarks.} 
\begin{enumerate}
\item The Theorem can be applied to a single cluster; in particular, some clusters may contract to tiny balls while others do not contract at all.
\item Since $\alpha h \sim n$, we see that the bound is only nontrivial if, for some small constant $c_2>0$, 
$$ \sum_{j \neq i \atop \mbox{\tiny other clusters}}{ p_{ij} } \leq \frac{c_2}{n}.$$
Otherwise, it merely tells us that the elements of the clusters are contained in a ball of radius $\sim 1$ (as are all the other points).
Generally, for well-clustered data, we would expect that sum to be very close to 0 which would yield a leading term error of $c h/n$.
\item The constant $c$ seems to be roughly on scale $c \sim 10$ for well-clustered data and slightly
larger for data with worse clustering properties (in particular, for the classical t-SNE parameter section, it would slowly increase with 
the number of points $n$). We believe this estimate to conservative and consider the true value to be on a smaller order of
magnitude; this question will be pursued in future work.
\end{enumerate}

The proof of the main result is actually rather versatile and should easily adapt to a variety of other settings that might be of interest.
This versatility is partly due to the connection of the argument to rather fundamental ideas in partial differential equations, indeed, the
argument may be interpreted as a maximum principle for a discrete parabolic operator acting on vector-valued (i.e. points in space)
data. This interpretation is what led us to establish a connection to spectral clustering which we now discuss.

\section{A Connection to Spectral Clustering}\label{sec:spectral}
\subsection{Approximating spectral clustering}\label{subsec:spectral}
The purpose of this section is to note that it is possible to take the limit
$\alpha \rightarrow \infty, ~h \rightarrow 0$ (scaled so that $\alpha \cdot h =
\mbox{const}$) and that, in that limit, one obtains a simple spectral
clustering method.  We re-introduce notation and assume again that $\mathcal{X}
= \left\{x_1, x_2, \dots, x_n \right\} \subset \mathbb{R}^d$ is given. We assume
$p_{ij}$ is some collection of affinities scaled in such a way that for $x_i,
x_j$ in the
same cluster $\pi(i) = \pi(j)$
$$ p_{ij} \geq \frac{1}{10} \frac{1}{|\pi^{-1}(\pi(i))|} \qquad \mbox{and} \qquad \sum_{j = 1}^{n}{p_{ij}} \leq 1.$$
We observe that this scaling is slightly differently than the one above: it is obtained by absorbing the $\alpha h \sim n$ term
into the affinities. At the same time, $h \rightarrow 0$ implies that the repulsion term containing the $q_{ij}$ does not exert any force.
This implies that, in the limit, the remaining term in the gradient descent method is given by
\begin{align*}
y_i(t+1) &= y_i(t) + \sum_{j \neq i}{p_{ij} (y_j(t) - y_i(t))} \\
&= \sum_{j \neq i}{p_{ij} y_j(t)} + \left(1 - \sum_{j \neq i}{p_{ij}}\right)y_i(t).
\end{align*}
This, however, can be interpreted as a Markov chain with suitably chosen transition probabilities. It may be unusual, at first, to see this
equation since the $y_i(t)$ are vectors in $\mathbb{R}^2$, however, all the equations separate different coordinates, which allows for
a reduction to the familiar form. All the canonical results from spectral clustering apply: the asymptotic behavior is given by the largest 
non-trivial eigenvalue(s), which are either 1 (in the case of perfectly separated clusters) or very close to 1 and convergence speed depends
on the spectral gap. 

\subsection{Visualizing spectral clustering}\label{sec:visualizing_spec_clust} The connection also allows us to go the other direction and discuss a particular visualization technique
for spectral methods that shows arising clusters as points in $\mathbb{R}^2$ (or higher dimensions, which is not essential here). The transition
matrix of the Markov chain is given by
$$ A_{ij} = \begin{cases}1 - \sum_{i \neq k}{p_{ik}}   \qquad &\mbox{if}~i = j\\
p_{ji} \qquad &\mbox{otherwise.}
\end{cases}$$
The large-time behavior of $y(t) = A^t y(0)$ is essentially determined by the spectrum of $A$ close to 1. 
Moreover, in the case of perfect clustering with $p_{ij} = 0$ whenever $x_i$ and $x_j$ are in different clusters, there are exactly $k$ eigenvalues
equal to 1 and the initialization converges to that. 

\begin{center}
\begin{figure}[h!]
\begin{tikzpicture}[scale=1]
\node at(0,1) {\includegraphics[width=\textwidth]{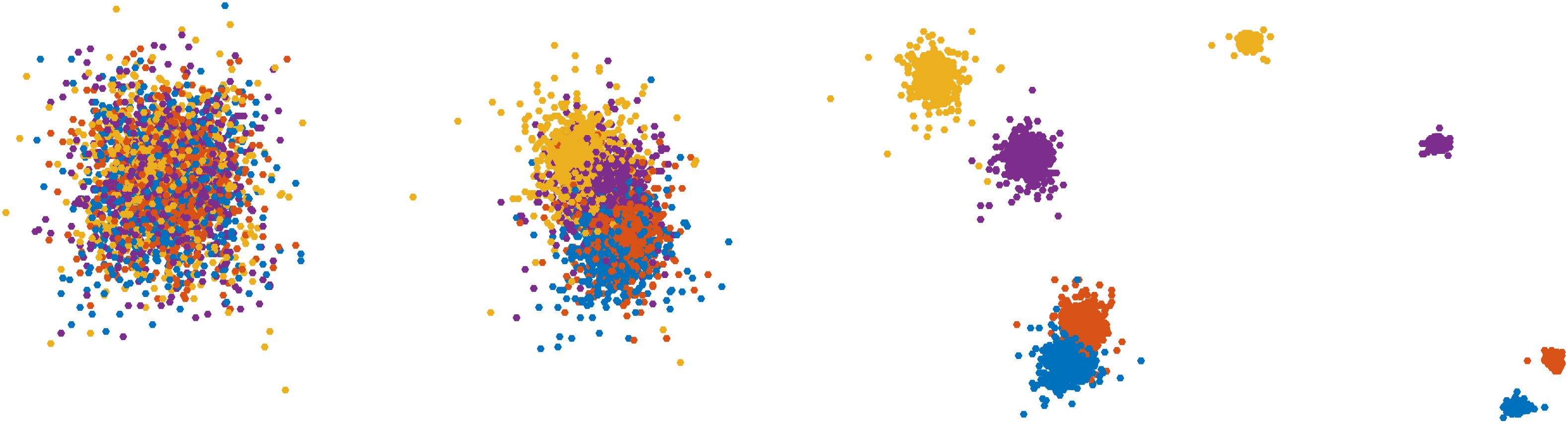}};
\node at (-5.8,3.2) {Step 50};
\node at (-1.8,3.2) {Step 100};
\node at (2,3.2) {Step 150};
\node at (5.5,3.2) {Step 200};
\node at(0,-3.5) {\includegraphics[width=\textwidth]{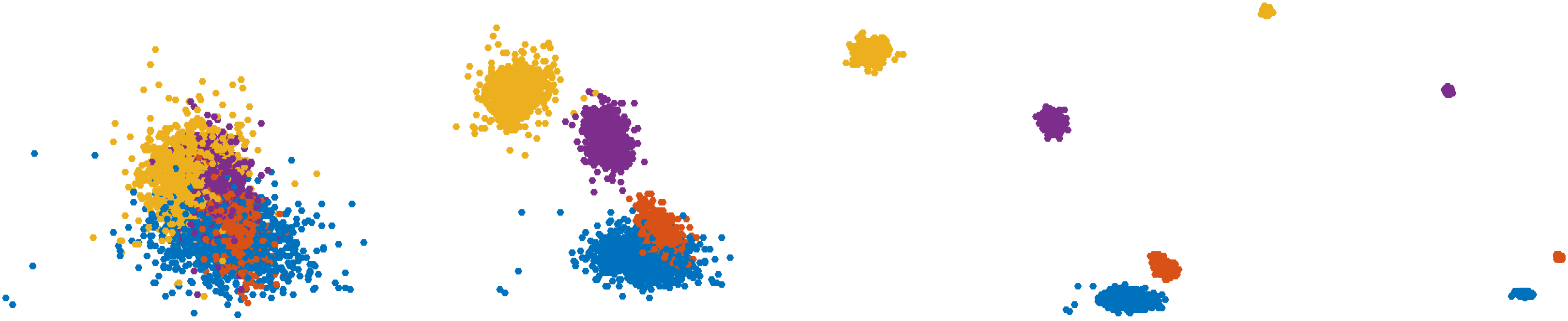}};
\node at (-5.6,-1.6) {Step 50};
\node at (-1.8,-1.6) {Step 100};
\node at (2.3,-1.6) {Step 150};
\node at (5.5,-1.6) {Step 200};
\end{tikzpicture}
\caption{Early exaggeration via t-SNE with $\alpha = n/10, h = 1$ (top) and visualization of iterations of the spectral method (bottom) on same initialization.}
\label{fig:spectral}
\end{figure}
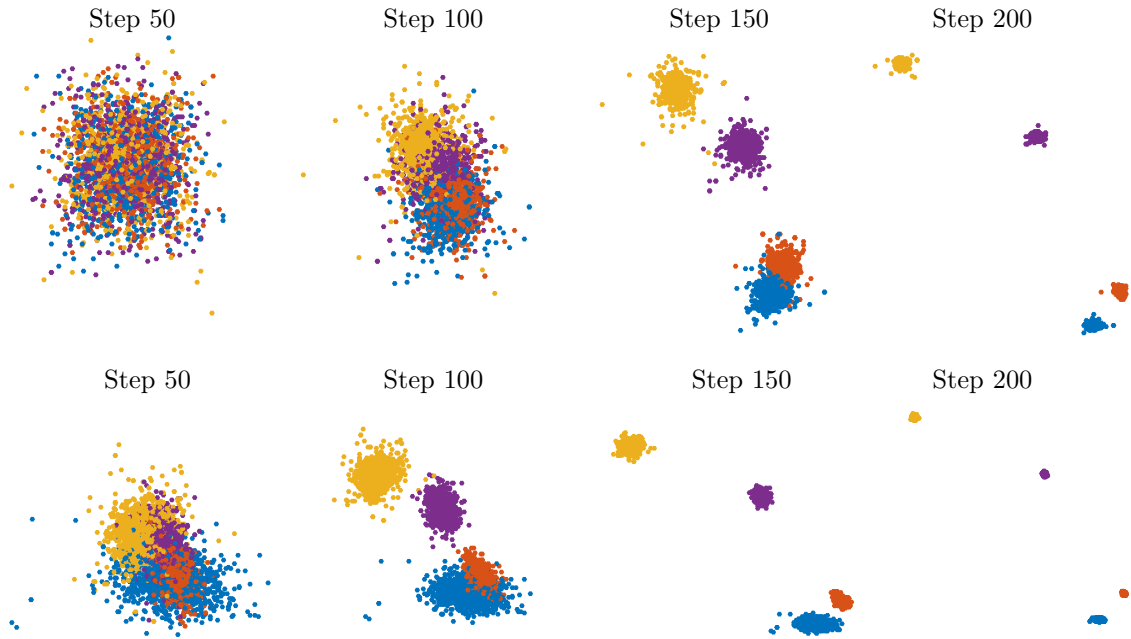
\end{center}
Let us now assume that the goal is visualization in $\mathbb{R}^2$. We let
$\mathcal{Y} = \left\{y_1(0), y_2(0), \dots, y_n(0)\right\} \subset
\mathbb{R}^2$ be a set of points that we assume are i.i.d. random variables
from, say, the uniform distribution on $[-0.01,0.01]^2$. 
We propose to visualize the point set after $k$ iterations as follows: collect these $n$ initial vectors in a $n \times 2$ vector $\underline{y}$
and interpret the $n$ rows of $A^k \underline{y}$ as coordinates in $\mathbb{R}^2$. This creates t-SNE-style visualizations for spectral methods (see Fig. \ref{fig:spectral}
and Fig. \ref{fig:mnist}, lower rows). \\

\textit{Examples.} An example of this method is shown in Figure \ref{fig:spectral}. The example is comprised of 40000 points in $\mathbb{R}^{25}$ sampled
from four very narrow Gaussians and are highly clustered. We used perplexity of 30 to create the $p_{ij}$ and used $\alpha = n/10, h=1$
in the implementation of t-SNE. The second row in Figure \ref{fig:spectral} shows the projection onto the 50 largest eigenvectors of $A$.
The computation time of t-SNE took roughly 7 minutes vs. 1 minute for the spectral decomposition -- note, however, that once
the spectral decomposition has been computed, then iterations can be computed in constant time (one only has to raise
the eigenvalues to some power).

\begin{center}
\begin{figure}[h!]
\begin{tikzpicture}[scale=1]
\node at(0,1) {\includegraphics[width=\textwidth]{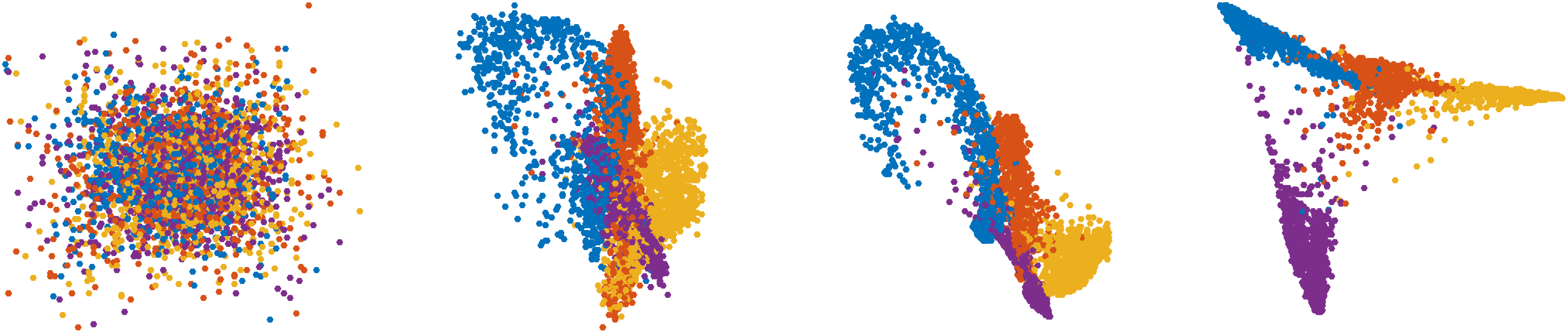}};
\node at (-5.8,3.2) {Step 5};
\node at (-1.8,3.2) {Step 50};
\node at (2,3.2) {Step 100};
\node at (5.5,3.2) {Step 800};
\node at(0,-3.5) {\includegraphics[width=\textwidth]{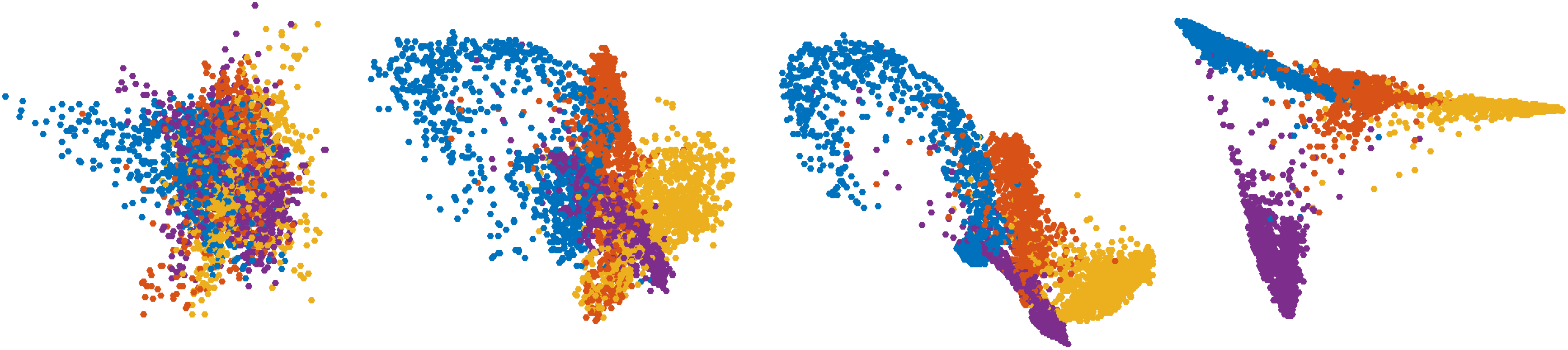}};
\node at (-5.6,-1.6) {Step 5};
\node at (-1.8,-1.6) {Step 50};
\node at (2.3,-1.6) {Step 100};
\node at (5.5,-1.6) {Step 800};
\end{tikzpicture}
\caption{Early exaggeration via t-SNE with $\alpha \sim n/3, h = 1$ (top, parameter selection via guideline) and iterations of the spectral method (bottom).}
\label{fig:mnist}
\end{figure}
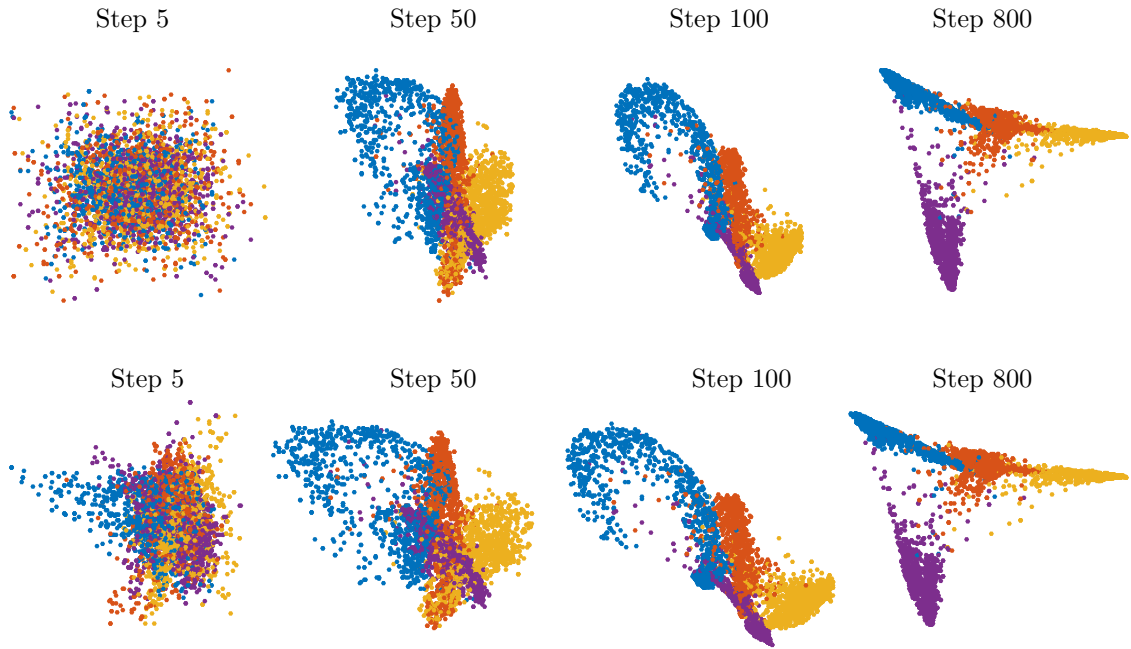
\end{center}
\vspace{-20pt}
Another example is given in Figure \ref{fig:mnist} that is run on 4 digits in MNIST; again, both methods coincide.
This shows that our derivation of the approximating spectral method was 
accurate. At the same time, it suggests to repeat the fundamental question.
\begin{quote}
\textbf{Open problem.} Is the clustering behavior of the early exaggeration phase of t-SNE with $\alpha = 12, h = 200$ (and possibly optimization techniques
such as momentum) essentially qualitatively
equivalent to the behavior of t-SNE with our parameter choice $\alpha = n/10, h = 1$?
\end{quote}
If this were indeed the case, then the early exaggeration phase of t-SNE would be simply a spectral clustering method in disguise;
if not, then it would be very valuable to understand under which circumstances its performance is superior to spectral clustering
and whether its underlying mechanisms could be used to boost spectral methods. We reiterate that we believe this to
be a very interesting problem.

\section{Ingredients for the Proof: Discretized Dynamical Systems}\label{sec:dynamical}
This section introduces a type of discrete dynamical systems on sets of points in $\mathbb{R}^s$ and we describe 
their asymptotic behavior; this is a self-contained result; it could potentially be interpreted as an analysis of a spectral method
that is robust to small error terms but the analysis is simple enough for us to keep entirely self-contained. Our original guiding
picture was that of the maximum principle in the theory of parabolic partial differential equations.

\subsection{A discrete dynamical system}
\label{section:system}
Let $z_1, \dots, z_n \in \mathbb{R}^s$ be given. We use them as initial values for a time-discrete dynamical system that is defined via
\begin{align*}
z_i(t+1) &= z_i(t) + \sum_{j =1}^{n}{\alpha_{i,j,t}(z_j(t) - z_i(t))} + \varepsilon_i(t)\\
z_i(0) &= z_i
\end{align*}
At this stage, if the points are in general position and $n \geq s$, basic linear algebra implies that this system can undergo almost any arbitrary
evolution as long as one is free to choose $\alpha_{i,j,t}$. We will henceforth assume that these parameters assume the
following three conditions.
\begin{enumerate}
\item There is a uniform lower bound on the coefficients for all $t>0$ and all $i \neq j$
$$ |\alpha_{i,j,t}| \geq \delta > 0.$$
\item There is a uniform upper bound on the coefficients
$$ \sum_{j =1}^{n}{\alpha_{i,j,t}} \leq  1.$$
\item There is a uniform upper bound on the error term
$$ \| \varepsilon_i(t) \| \leq \varepsilon.$$
\end{enumerate}

A typical example of such a dynamical system is given in the Figure below: we start with twelve points on the unit circle and then iterate the system 
for some random choices of $a_{i,j,t}$ and random $\varepsilon_i(t)$. The points move at first
towards each other until they are close and the error term starts being on the same scale as the forces of attraction. The points then move around
randomly (all the while staying close to each other). We will make this intuitive picture precise below.
\begin{center}
\begin{figure}[h!]
\includegraphics[width=0.45\textwidth]{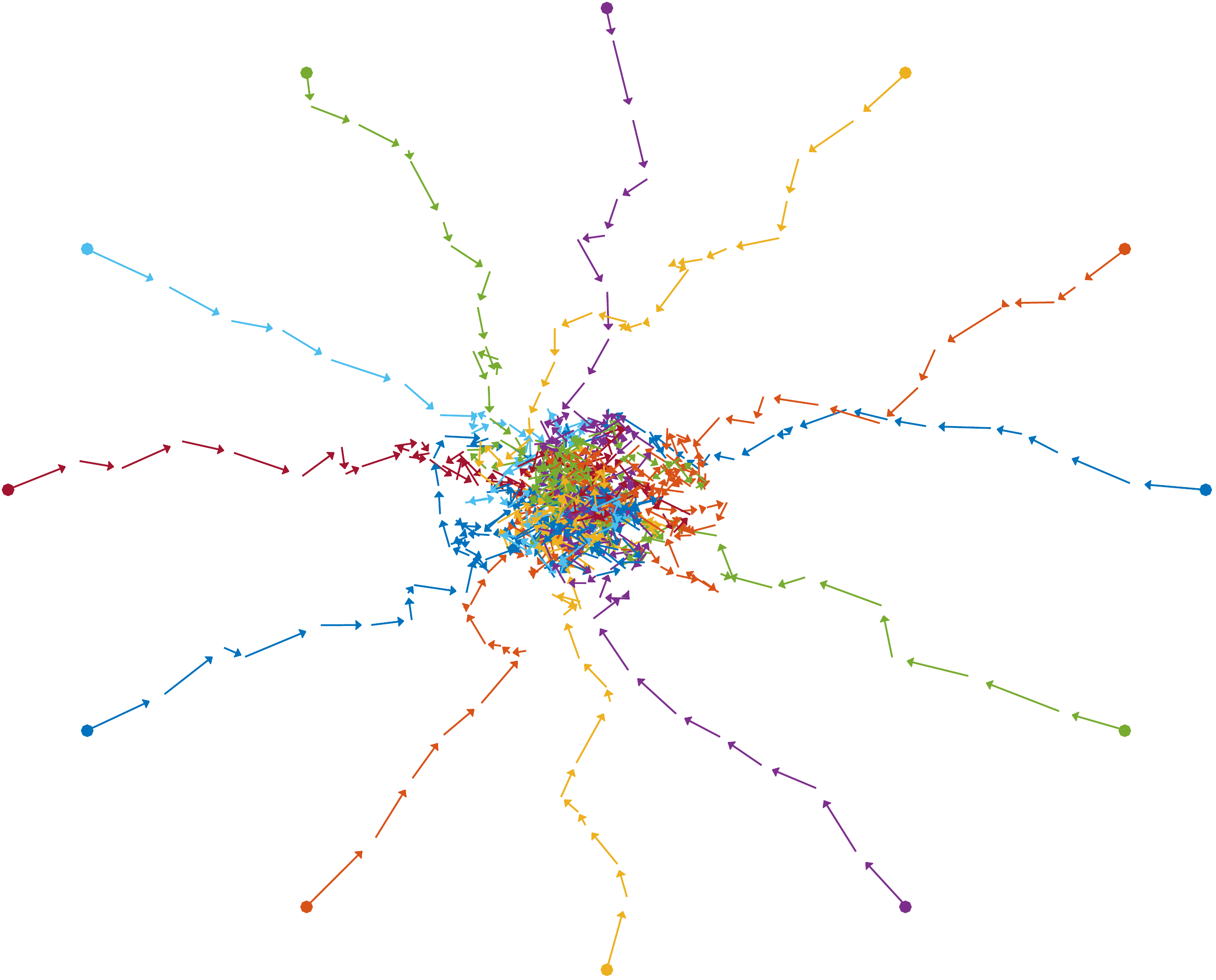}
\label{fig:spider}
\caption{A typical evolution a dynamical system of this type.}
\end{figure}
\end{center}
\vspace{-10pt}
The main result of this section is that all the points in this dynamical system are eventually contained
in a ball whose size only depends on $n, \delta$ and $\varepsilon$.
We start by showing that the convex hull of the points is stable. We use $B(0,\varepsilon)$ to denote 
a ball of radius $\varepsilon$, $A+B = \left\{a + b: a \in A \wedge b \in B \right\}$ and $\conv{A}$ 
for the convex hull of $A$.

\begin{lemma}[Stability of the convex hull] With the assumptions above, we have
$$ \conv  \left\{z_1(t+1), z_2(t+1), \dots, z_n(t+1) \right\}  \subseteq  \conv  \left\{z_1(t), z_2(t), \dots, z_n(t) \right\}  + B(0, \varepsilon),$$
\end{lemma}
\begin{proof} This argument is simple. We note that
\begin{align*}
z_i(t+1) &= z_i(t) + \sum_{j =1}^{n}{\alpha_{i,j,t}(z_j(t) - z_i(t))} + \varepsilon_i(t) \\
&= \left(1 - \sum_{j=1 \atop j \neq i}^{n}{ \alpha_{i,j,t}} \right) z_i (t) +  \sum_{j =1 \atop j \neq i}^{n}{\alpha_{i,j,t} z_j(t)} + \varepsilon_i(t).
\end{align*}
By assumption,
$$  0 \leq \sum_{j =1}^{n}{\alpha_{i,j,t}} \leq 1$$
and this implies $z_i(t+1) - \varepsilon_i(t) \in  \conv  \left\{z_1(t), z_2(t), \dots, z_n(t) \right\}$.
\end{proof}

\begin{lemma}[Contraction inequality] With the notation above, if the diameter is large
$$ \diam \left\{z_1(t), z_2(t), \dots, z_n(t) \right\} \geq \frac{10\varepsilon}{n \delta},$$
then
$$ \diam \left\{z_1(t+1), z_2(t+1), \dots, z_n(t+1) \right\} \leq \left(1 - \frac{n \delta}{20} \right) \diam \left\{z_1(t), z_2(t), \dots, z_n(t) \right\}.$$
\end{lemma}
One particularly important consequence is the following:  the diameter shrinks, at an exponential rate $\left(1 - n \delta /20 \right)^t$, to a size of $\sim \varepsilon/(n\delta)$. Of course, this
convergence is particularly fast whenever $n \delta \sim 1$. It is easy to see,
for example by taking $n=2$ points in $\mathbb{R}$, that this is the optimal scale for the result to hold.

\begin{proof}[Proof of Lemma 2] 
The method of proof will be as follows: we will project the set of points onto an arbitrary line (say, the $x-$axis by taking only the first coordinate of each point) and show that the one-dimensional
projections contract exponentially quickly. This then implies the desired statement. Let $\pi_x:\mathbb{R}^n \rightarrow \mathbb{R}$ be such a projection. We abbreviate
the diameter of the projection as 
$$ \diam :=  \diam \left\{\pi_x z_1(t), \pi_x z_2(t), \dots, \pi_x z_n(t) \right\}.$$
We may assume w.l.o.g. that this set is contained in $\left\{\pi_x z_1(t), \pi_x z_2(t), \dots, \pi_x z_n(t) \right\} \subset [0,\diam]$. We then subdivide the
interval into two regions
$$ I_1 = \left[0, \frac{\diam}{2}\right] \quad \mbox{and} \quad I_2 = \left(\frac{\diam}{2}, \diam \right]$$
and denote the number of points in each interval by $i_1, i_2$. Clearly, $i_1 + i_2 = n$ and therefore either $i_1 \geq n/2$ or $i_2 \geq n/2$. We assume w.l.o.g.
the first case holds. 
Projections are linear, thus
$$
 \pi_x z_i(t+1) = \pi_x z_i(t) + \sum_{j=1}^{n}{ a_{i,j,t} \pi_x (z_j(t) - z_i(t)) } + \pi_x \varepsilon_i(t).
$$
We abbreviate 
$$ 0 \leq \sigma := \sum_{j =1}^{n}{\alpha_{i,j,t}} \leq 1$$
and write
\begin{align*}
 \sum_{j=1}^{n}{ a_{i,j,t} \pi_x (z_j(t) - z_i(t)) } &=  \sum_{\pi_x z_j \leq \diam/2}^{n}{ a_{i,j,t} \pi_x (z_j(t) - z_i(t)) } + \sum_{\pi_x z_j > \diam/2}^{n}{ a_{i,j,t} \pi_x (z_j(t) - z_i(t)) }\\
&\leq \sum_{\pi_x z_j \leq \diam/2}^{n}{ a_{i,j,t} \left(  \frac{\diam}{2} - \pi_x z_i(t) \right) } + \sum_{\pi_x z_j > \diam/2}^{n}{ a_{i,j,t} \left( \diam -\pi_x z_i(t) \right)}\\
&= \sum_{\pi_x z_j \leq \diam/2}^{n}{ a_{i,j,t}   \frac{\diam}{2}  } + \sum_{\pi_x z_j > \diam/2}^{n}{ a_{i,j,t}  \diam}  - \sigma \pi_x z_i(t).
\end{align*}
Moreover, using the lower bound $a_{i,j,t} \geq \delta$
\begin{align*}
 \diam \left( \frac{1}{2} \sum_{\pi_x z_j \leq \diam/2}^{n}{ a_{i,j,t}} + \sum_{\pi_x z_j > \diam/2}^{n}{ a_{i,j,t}} \right) &\leq \diam \left( \frac{1}{2} \sum_{\pi_x z_j \leq \diam/2}^{n}{ a_{i,j,t}} + \left(\sigma -\sum_{\pi_x z_j \leq \diam/2}^{n}{ a_{i,j,t}}\right) \right)\\
&= \diam \left( \sigma - \frac{1}{2}\sum_{\pi_x z_j \leq \diam/2}^{n}{ a_{i,j,t}} \right)\\
&\leq \left(\sigma - \frac{n\delta}{4}\right) \diam.
\end{align*}
Then, however,
\begin{align*}
 \pi_x z_i(t+1) &= \pi_x z_i(t) + \sum_{j=1}^{n}{ a_{i,j,t} \pi_x (z_j(t) - z_i(t)) } \leq (1-\sigma)  \pi_x z_i(t) +  \left(\sigma - \frac{n\delta}{4}\right) \diam \\
&\leq (1-\sigma) \diam +  \left(\sigma - \frac{n\delta}{4}\right) \diam =  \left(1 - \frac{n\delta}{4}\right) \diam,
\end{align*}
which shows that $\pi_x z_i(t+1)\in [0, \diam (1-n \delta/4)].$ Accounting for the error term, we get
$$  \diam \left\{\pi_x z_1(t+1), \pi_x z_2(t+1), \dots, \pi_x z_n(t+1) \right\} \leq \left(1 - \frac{n\delta}{4}\right) \diam + 2\varepsilon.$$
If the diameter is indeed disproportionately large
$$ \diam \geq \frac{10\varepsilon}{n \delta},$$
then this can be rearranged as
$$ \varepsilon \leq \frac{n \delta}{10} \diam $$
and therefore
\begin{align*}
 \left(1 - \frac{n\delta}{4}\right) \diam + 2\varepsilon \leq \left(1 - \frac{n\delta}{4}\right) \diam + \frac{n \delta}{5} \diam \leq  \left(1 - \frac{n\delta}{20}\right) \diam.
\end{align*}
Since this is true in every projection, it also holds for the diameter of the original set.
\end{proof}

\textit{Remark.} The argument could be slightly improved because in its current form it
assumes that the error $\varepsilon_i$ has $\| \varepsilon_i(t)\|_{\ell^{\infty}} = \varepsilon$, while we assume $\| \varepsilon_i(t)\|_{\ell^{\infty}} \leq \| \varepsilon_i(t)\|_{\ell^{2}} = \varepsilon$.
This, together with the usual other optimization schemes, should yield an improved estimate on the constant.
 The condition on $\delta$ could also be weakened (at the cost of losing constants). In particular, it would be sufficient in Assumption (1) in our main result to assume that, for every $1 \leq i \leq n$ 
$$ p_{ij} \geq  \frac{1}{10 n | \pi^{-1}(\pi(i))| } ~\mbox{holds for at least} \quad \left(\frac{1}{2} + \varepsilon\right)  | \pi^{-1}(\pi(i))| \quad \mbox{values of} ~j$$
that are in the same cluster $\pi(z_i) = \pi(z_j)$. By adapting the proof, the constant $(1/2 + \varepsilon)$ could be reduced further, however, this is inevitably going to decrease the provable
bounds on the exponential decay rate (which is not an artifact of the method, convergence will slow down).

\section{Proof of the Main Result}\label{sec:proof_main_result}

The rough outline of the argument is as follows: we initialize all points inside $[-0.01,0.01]^2$. We rewrite the gradient
descent method acting on one particular embedded cluster as a dynamical system of the type studied above with an error term. The error term contains $q_{ij}$, which 
depend on distances between points from different clusters. This is difficult to control, especially if the points are far apart. Our strategy will now be as follows: we 
show that the $q_{ij}$ are all under control as long as everything is contained in $[-0.02,0.02]^2$. We use stability of the convex hull to guarantee that all of the embedded
points are within $[-0.02,0.02]^2$ for at least $\ell$ iterations and show that this time-scale is enough to guarantee contraction of the cluster.

\begin{proof}  We start by showing that the $q_{ij}$ are comparable as long as the point set is contained in a small region space.
Let now $\left\{y_1, y_2, \dots, y_n\right\} \subset [-0.02,0.02]^2$ and recall the definitions
$$   q_{ij} = \frac{(1 + \| y_i -  y_j\|^2)^{-1}}{\sum_{k\neq l} (1 +\|  y_k -  y_l\|^2 )^{-1}} \qquad \mbox{and} \qquad Z=\sum_{k\neq l} (1 +\|  y_k -  y_l\|^2 )^{-1}.$$
Then, however, it is easy to see that $0 \leq \|y_i - y_j\| \leq 0.06$ implies
$$  \frac{9}{10} \leq  q_{ij}Z =  (1 + \| y_i -  y_j\|^2)^{-1} \leq 1.$$
We will now restrict ourselves to a small embedded cluster $\left\{ y_i: \pi(i)~\mbox{fixed}\right\}$ and rewrite the gradient descent method as
\begin{align*}
y_i(t+1) = y_i(t) &+\sum_{j \neq i \atop \mbox{\tiny same cluster}}{  (\alpha h) p_{ij} q_{ij} Z (y_j(t) - y_i(t))} \\
&+  \sum_{j \neq i \atop \mbox{\tiny other clusters}}{  (\alpha h) p_{ij} q_{ij} Z (y_j(t) - y_i(t))} \\
&- h \sum_{j \neq i }{ q_{ij}^2 Z (y_j(t) - y_i(t))},
\end{align*}
where the first sum is yielding the main contribution and the other two sums are treated as a small error. Applying our results for dynamical systems of this type requires us to
verify the conditions. We start by showing the conditions on the coefficients to be valid. Clearly,
$$  \alpha h p_{ij} q_{ij} Z  \geq  \alpha h p_{ij} \frac{9}{10} \geq  \frac{\alpha h}{10 n |\pi^{-1}(\pi(i))|}\frac{9}{10} \geq  \frac{9}{100} \frac{\alpha h}{n} \frac{1}{|\pi^{-1}(\pi(i))|} \sim \delta,$$
which is clearly admissible whenever $\alpha h \sim n$. As for the upper bound, it is easy to see that
$$  \sum_{j \neq i \atop \mbox{\tiny same cluster}}{  (\alpha h) p_{ij} q_{ij} Z} \leq  \alpha h \sum_{j \neq i \atop \mbox{\tiny same cluster}}{  p_{ij} } \leq 1.$$
It remains to study the size of the error term for which we use the triangle inequality
\begin{align*}
 \left\| \sum_{j \neq i \atop \mbox{\tiny other clusters}}{  (\alpha h) p_{ij} q_{ij} Z (y_j(t) - y_i(t))} \right\|  &\leq   \alpha h  \sum_{j \neq i \atop \mbox{\tiny other clusters}}{ p_{ij}  \left\|y_j(t) - y_i(t)\right\|}  \\
&\leq 0.06 \alpha h \sum_{j \neq i \atop \mbox{\tiny other clusters}}{ p_{ij} }
\end{align*}
and, similarly for the second term,
\begin{align*}
\left\|  h \sum_{j \neq i }{ q_{ij}^2 Z (y_j(t) - y_i(t))}  \right\|  \leq  h   \sum_{j \neq i }{ q_{ij} \left\| (y_j(t) - y_i(t))\right\| }  \leq 0.06h  \sum_{j \neq i }{ q_{ij}}  \leq \frac{0.1h}{n}.
\end{align*}
This tells us that the norm of the error term is bounded by
$$ \| \varepsilon\| \leq 0.1 h \left( \alpha  \sum_{j \neq i \atop \mbox{\tiny other clusters}}{ p_{ij} } + \frac{1}{n}\right).$$
It remains to check whether time-scales fit. The number of iterations $\ell$ for which the assumption  $\mathcal{Y} \subset [-0.02,0.02]^2$ is reasonable is at least $ \ell \geq 0.01/\varepsilon.$
At the same time, the contraction inequality implies that in that time the cluster shrinks to size
$$ \max\left\{ \frac{10 \varepsilon}{|\pi^{-1}(\pi(i))|  \delta}, 0.01\left(1- \frac{1}{20}\right)^{\ell}\right\} \leq  \max\left\{ \frac{10 \varepsilon}{|\pi^{-1}(\pi(i))| \delta},  8 \varepsilon \right\},$$
where the last inequality follows from the elementary inequality
$$ \left(1 - \frac{1}{20}\right)^{1/100\varepsilon} \leq 8 \varepsilon.$$
\end{proof}

\textit{Remarks.} The proof is relatively flexible in several different spots. By demanding that the initialization $\mathcal{Y}$ is contained in a sufficiently small ball, one can force the quantity $Z q_{ij}$ to be arbitrarily close to 1. We also emphasize that we did not optimize over constants and additional fine-tuning in various spots would yield better constants (at the cost of a more involved argument which is why
we decided against it). The use of the triangle inequality in bounding the error terms is another part of the proof that deserves attention: if the clusters are spread out, then we would
expect the repulsive forces to act from all directions and lead to additional cancellation (which, if it were indeed the case that the $q_{ij}$ do not play a significant role in the clustering that occurs
in the early exaggeration phase, would be an additional reason for the strong similarity to the outcome of the spectral method). It could be of interest to study mean-field-type approximations
to gain a better understanding of this phenomenon.

\bibliographystyle{amsplain}

\end{document}